\documentclass[runningheads]{llncs}

\usepackage[preprint,year=2024,ID=3126]{eccv}

\usepackage{eccvabbrv}

\usepackage{graphicx}
\usepackage{booktabs}

\usepackage[pagebackref,breaklinks,colorlinks]{hyperref}

\usepackage{orcidlink}
\newcommand\blue{\textcolor{blue}}

\newcommand\red{\textcolor{red}}

\usepackage{float}

\usepackage{booktabs}
\usepackage{colortbl}
\usepackage{placeins}

\definecolor{mygray}{gray}{.9}
\newcommand{\indep}{\perp\!\!\!\perp}

\begin{document}

\title{Revisiting Information Maximization for Generalized Category Discovery}

\titlerunning{Revisiting Mutual Information Maximization for Generalized Category Discovery}

\author{Zhaorui Tan$^*$\inst{1,2}
    \and
    Chengrui Zhang$^*$\inst{1,2}
     \and
    Xi Yang\inst{1} 
    \and
    Jie Sun\inst{1} \\
    \and
    Kaizhu Huang$^\dag$\inst{3}
    }

\authorrunning{F.~Author et al.}

\institute{
    Xi'an-Jiaotong Liverpool University 
    \and
    Liverpool University 
    \and
    Duke Kunshan University
    }

\maketitle

\begin{abstract}
    Generalized category discovery presents a challenge in a realistic scenario, which requires the model's generalization ability to recognize unlabeled samples from known and unknown categories.
    This paper revisits the challenge of generalized category discovery through the lens of information maximization (InfoMax) with a probabilistic parametric classifier. Our findings reveal that ensuring \textit{independence} between known and unknown classes, {while concurrently assuming a \textit{uniform probability distribution} across all classes,} 
    yields an enlarged margin among known and unknown classes
    that promotes the model's performance. 
    To achieve the aforementioned independence, we propose a novel InfoMax-based method, 
    \textbf{R}egularized \textbf{P}arametric \textbf{I}nfo\textbf{M}ax (RPIM), which adopts pseudo labels to supervise unlabeled samples during InfoMax, while proposing a regularization to ensure the quality of the pseudo labels. 
    Additionally, we introduce novel semantic-bias transformation to refine the features from the pre-trained model instead of direct fine-tuning to rescue the computational costs.  
    Extensive experiments on six benchmark datasets validate the effectiveness of our method. RPIM significantly improves the performance regarding unknown classes, surpassing the state-of-the-art method by an average margin of $3.5\%$. 
\end{abstract}





\section{Introduction}
\label{sec:intro}


\begin{figure}[t]
    \centering
    \includegraphics[width=0.99\linewidth]{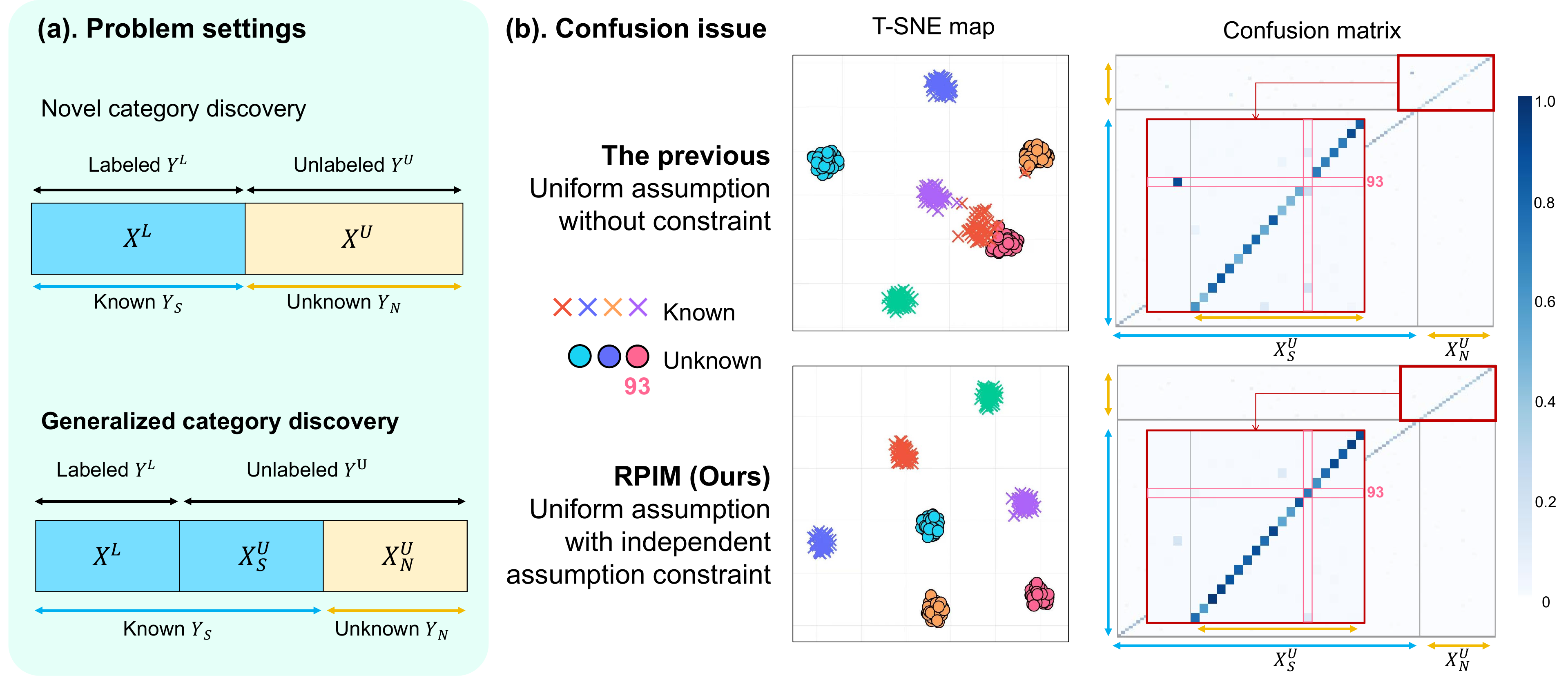}
    \caption{
    (a). Diagram of problem settings. 
    (b) Visualization of the confusion issue on CIFAR100~\cite{krizhevsky2009learning}. Left:  T-SNE map of latent features $Z$ from those classes. Right: Confusion matrix of unlabeled set between known and unknown classes.
    {Solely satisfying the uniform assumption for unconfident predictions causes confusion issues, while our proposed RPIM effectively mitigates.}
    }
    \label{fig:vis_dist}
\end{figure}

Category discovery significantly broadens the application scope of visual recognition by considering scenarios with less human supervision or limited predefined categories; that is, some training samples are unlabeled and may correspond to \textit{unknown classes}, i.e., those not yet defined. This task is particularly crucial for processing large and complex datasets where manual annotation of all categories is impractical, for instance, recognizing products in a supermarket, pathologies in medical
images, vehicles in autonomous driving, etc. Recent efforts~\cite{han2019learning,han2021autonovel,fini2021unified,cao2022openworld} leverage knowledge from known classes to enhance the clustering of unlabeled samples that belong solely to unknown classes, thus establishing the novel category discovery task. As the complexity of the scenario further escalates, the task of generalized category discovery~\cite{vaze2022generalized} relaxes the disjoint assumption between labeled and unlabeled categories. This introduces the challenge of an unlabeled set composed of samples from both known and unknown classes (see the diagram of problem settings in \cref{fig:vis_dist}~(a)).

Existing wisdom for tackling the generalized category discovery bifurcates into two distinct branches, depending on their use or non-use of probabilistic parametric classifiers. Without probabilistic parametric classifiers, popular efforts such as
GCD~\cite{vaze2022generalized} leverage contrastive training and semi-supervised method. However, these approaches often fail to address class imbalances, yielding sub-optimal results on long-tailed datasets. In contrast, recent innovations adopting probabilistic parametric classifiers, such as 
PIM~\cite{chiaroni2023parametric}, endeavor to maintain maximum information from the input for the prediction guided by the \textit{InfoMax principle}~\cite{linsker1988self}. PIM validates the view that InfoMax implicitly enables probabilistic parametric classifiers to maintain uniform class proportions~\cite{krause2010discriminative,boudiaf2020information}. It argues that the uniform assumption\footnote{{The probability of classes is assumed to 
follow the uniform distribution.
}} 
{should be applied to unconfident predictions}.
As such, it effectively mitigates the class-balance bias encoded in standard InfoMax, achieving superior performance across both short-tailed and long-tailed datasets. 

Inspired by the aforementioned perspective, we conduct a series of theoretical validations, {revealing that both the \textit{uniform assumption}
and \textit{independence assumption}\footnote{Known and unknown classes are assumed to be independent of each other.}} are essentially indispensable for generalized category discovery. This entails the presumption that the probability distribution of unconfident known and unknown classes 
{tends to be uniform,}
and mandates an additional independence constraint between known and unknown classes. In the absence of the independence assumption, a ``\textit{confusion issue}'' may arise, manifested as an unclear margin between known and unknown classes, 
probably due to the absence of 
{extra guidance of unlabeled sets.}
This phenomenon can also be observable experimentally (see \cref{fig:vis_dist}~(b)), where, without constraints on the independence, the model confuses the unknown class $93$  with the known classes.

Taking into account the confusion issue,
this paper revisits the problem of generalized category discovery within the scope of InfoMax. 
As one major contribution, we theoretically and empirically validate that integrating reliable pseudo labels as {additional supervision} for unlabeled samples actually fulfills the independence assumption between known and unknown classes. 
Based on this finding, we also propose a regularization term that leverages given labels as anchors for pseudo labels, aiming to reduce their overall empirical risk. 
To further improve the probabilistic parametric classifier and reduce the computational cost, instead of fine-tuning directly, we introduce a simple yet effective semantic-bias transformation to refine the features from the pre-trained model. 
 Overall, these ideas are integrated into a novel approach, termed as \textbf{R}egularized \textbf{P}arametric \textbf{I}nfo\textbf{M}ax (RPIM), which is specially designed for generalized category discovery. 

The key contributions are summarized as follows:
\textbf{(1).}~We revisit the InfoMax for generalized category discovery, revealing that ignoring the independence assumption may lead to a confusion issue, especially when adopting probabilistic parametric classifiers with the uniform assumption {for unconfident predictions}.
\textbf{(2).}~ Built upon theoretical insights,  we utilize given labels as anchors to obtain high-quality pseudo labels to supervise unlabeled samples, ensuring the independence assumption during the InfoMax process.
\textbf{(3).}~We propose a simple yet effective semantic-bias transformation to refine semantic features, able to reduce computational complexity without extensive fine-tuning of the entire pre-trained model. 
\textbf{(4).} Through extensive experiments on six datasets that cover short- and long-tailed distributions, our approach demonstrates high effectiveness. 
Compared to the SOTA PIM, our method shows notable improvements of $10.7\%$, $5.3\%$, and $3.7\%$ in unknown classes on CIFRA100~\cite{krizhevsky2009learning}, CUB~\cite{wah2011caltech}, and Stanford Cars~\cite{krause20133d}, respectively.



\section{Related work}
\label{sec:rw}
\textbf{Mutual information maximization.}
The well-known principle of InfoMax is presented in~\cite{linsker1988self}.  Its probabilistic theoretical base allows for more flexibility and has been widely validated in tasks such as clustering~\cite{krause2010discriminative,hu2017learning,jabi2019deep,info3d,hjelm2018learning} and few-shot learning~\cite{boudiaf2020information,veilleux2021realistic,boudiaf2021few,NEURIPS2020_196f5641}. 
Specifically, Krause~\etal~\cite{krause2010discriminative}  demonstrate 
that an InfoMax-based probabilistic parametric classifier allows
the specification of prior assumptions about expected class proportions. 
This paper unveils that this feature benefits generalized category discovery as the probability distribution of known and unknown classes is assumed to be uniform.
Additionally,
Tschannen~\etal~\cite{TschannenDRGL20} argue that the effectiveness of these methods cannot be attributed to the properties of mutual information, but also strongly depend on the inductive bias in both feature extractor architectures and the parametrization of mutual information estimators. 
Therefore, we introduce an effective transformation to promote the probabilistic parametric classifier.


\textbf{Novel category discovery.}
The novel category discovery task established in DCT~\cite{han2019learning} assumes that unlabeled samples belong to unknown categories. 
RankStats+~\cite{han2021autonovel}
addresses the novel category discovery task problem through a three-stage method that transfers low and high-level knowledge from labeled to unlabeled data with ranking statistics.
UNO+~\cite{fini2021unified} proposes a unified cross-entropy loss,
training the model simultaneously on labeled and unlabeled data by swapping the pseudo labels from their classification heads.
novel category discovery methods can be revised as generalized category discovery~\cite{vaze2022generalized}, which, however, may not guarantee acceptable performance since they assume that the categories of unlabeled samples are all unknown. 

\textbf{Generalized category discovery.}
Extending novel category discovery,
generalized category discovery is first proposed in GCD~\cite{vaze2022generalized} in which the unlabeled sample contains both known and unknown classes. 
ORCA~\cite{cao2022openworld} presents a similar open-world semi-supervised learning problem, tackling it by containing the intra-class variance of known and unknown classes
during training. This paper excels at the generalized category discovery problem.
In particular, GCD tries to estimate the number of categories in unlabeled data. 
It also proposes an approach that consists of contrastive training and 
a semi-supervised k-means-based clustering algorithm. 
However, without a probabilistic parametric classifier,
GCD requires optimizing all parameters in the pre-trained encoder, which is computationally consuming.
Additionally, without consideration of the possible unbalanced classes, GCD yields sub-optimal results on long-tailed data.
Another notable work is PIM~\cite{chiaroni2023parametric}, which introduces InfoMax into generalized category discovery for the first time. 
PIM maximizes the mutual information between the inputs and predictions while aligning labeled samples' predictions with given labels.
Unlike GCD, PIM freezes the encoder and trains only a one-layer probabilistic parametric classifier, significantly reducing computational consumption.  Moreover, its introduced losses are able to cope with imbalanced datasets, 
outperforming GCD on both short- and long-tailed datasets.
This paper also follows the InfoMax principle but reveals that PIM partially minimizes the overall empirical risk during InfoMax.

\section{Methodology}
\label{sec:rethink}
\textbf{Problem settings and preliminaries.}
The overall generalized category discovery setting is shown in \cref{fig:vis_dist}~(a).
Superscripts $*^{L}, *^{U}$ are utilized to denote the association of a variable with the labeled and unlabeled set, respectively. Subscript notation, $*_{S}$ for variables associated with \emph{known} classes and  $*_{N}$ for those aligned with \emph{unknown} classes, is adopted. 
$D$, $X$, and $Y$ denotes the dataset, data samples, and one-hot labels respectively. 

Consider a dataset consisting of the labeled and unlabeled sets,$D = D^L \cup  D^U$. 
$D$ consists of a total of $K$ classes, $|Y^L \cup  Y^U | = K$ and $|Y^L| < K$. Specifically, 
$Y^U_N \cap  Y^L = \emptyset$, $Y^U_S \subset  Y^L$, and $|Y^U_N \cup  Y^U_S| = K$.
For the data sample 
$X_i \in D$, it belongs to a class $y_i$ where $y_i = (y_{i,k})_{k \in \{1, ..., K^L \}}$ and $K^L < K$. 
Under this setting, the Generalized Category Discovery task~\cite{vaze2022generalized} is introduced, aiming to assign unlabeled samples to one of the known and unknown classes. This task joint a semi-supervised classification task for the known classes and a clustering task for the novel classes.

Following~\cite{chiaroni2023parametric},
the inputs that we used are latent features of raw samples produced by pre-trained models.
Thus, for notation simplification, we denote that $X \in R^D$ are the random variables from the latent features space, which belong to a pre-trained encoder, mapping the raw inputs to a latent feature with $D$ dimensions.
{We denote a model $f_\theta: X \to Z \in [0,1]^K$ with trainable parameters $\theta$. 
Note that $f_\theta$ consists of possible transformations and a parametric probabilistic classifier. 
$Z = Z^L\cup Z^U$ is the set of logits of labeled and unlabeled data samples, and $arg\max(Z)$ is the final label prediction.} 
Specifically, $z_i$ represents the predicted probability under all classes of the $i^{th}$ sample;
$H(\cdot)$, $H_c(\cdot, \cdot)$, $I(\cdot;\cdot)$ and $P(\cdot)$ represent entropy, cross-entropy, mutual information, and probability.

\subsection{Revisiting InfoMax for generalized category discovery}
\label{sec:Our approach}




This section intrinsically revisits the problem of InfoMax for generalized category discovery, which leads to the derivation of our proposed RPIM.
First, we presuppose the validity of the uniform and independence assumptions (i.e., {$P(y_i=k)_{k \in \{1,...,K\}} = \frac{1}{K}$} and $Y^U_S\indep Y^U_N$). 
The problem is formulated by taking $Y$ into InfoMax, aiming to find a proper $Z$ where the mutual information among $X, Y$, and $Z$ is maximized. 
Consequently, the overall objective is given as follows:
\begin{equation}
    \label{eq:our_I}
    \max_{\theta} I(X;Y;Z)=I(X;Z) +  I(Y;Z|X).
\end{equation}
Here, the two balance weights are omitted for simplicity. 
Then, $I(Y; Z|X)$ can be expanded by splitting $Z$ into $Z^L$ and $Z^U$:
\begin{equation}
    \label{eq:our_I_expand}
    \max_{\theta} I(X;Y;Z) =I(X;Z)+   I(Y^L; Z^L|X^L) + I(Y^U; Z^U|X^U).
\end{equation}
$I(X; Z) + I(Y^L; Z^L|X^L)$ represents the InfoMax between $X$ and $Z$ 
while constraining the predictions of the labeled set with ground-truth labels. 
\cref{sec:conn} shows that this part has already been studied in previous works, such as PIM~\cite{chiaroni2023parametric}, we make a further study on the previously unexplored term $I(Y^U; Z^U|X^U)$.

Under the premise of the independence assumption $Y^U_S\indep Y^U_N$, $I(Y^U; Z^U|X^U)$ is reformulated as:
\begin{equation}
    \label{eq:indep}
     I(Y^U; Z^U|X^U) = I(Y^U_S; Z^U_S|X^U_S) + I(Y^U_N; Z^U_N|X^U_N) ,
\end{equation}
indicating that maximizing $I(Y^U; Z^U|X^U)$ would enforce the precise and certain prediction for known and unknown classes in the unlabeled set.   This motivates us to seek an empirical form for maximizing $I(Y^U; Z^U|X^U)$.

As the ground-truth labels for ${Z}^U$ are unavailable during training, 
we adopt the Sinkhorn–Knopp algorithm \cite{cuturi2013sinkhorn} to produce soft pseudo labels, i.e., 
$\hat{Y} = Sinkhorn(Z)$.
In $\hat{Y}^U$, $\hat{Y}^U_S$ and $\hat{Y}^U_N$ are naturally independent of each other for discrete classification, which meets the requirements in \cref{eq:indep}.
However, not all pseudo labels are feasible for maximizing $I(Y^U; Z^U|X^U)$, 
due to the possible growth of uncertainty in $\hat{Y}^U$ during optimization.
Thus, only the maximum value in $\hat{y}_{i,k}\in \hat{Y}^U$ 
above the threshold $\mathcal{T}$ are feasible for \cref{eq:our_I_expand}:
\begin{equation}
\label{eq:pesudo_y}
    \hat{Y}^U_{T} =  \hat{Y}^U_{S,{T}} \cup \hat{Y}^U_{N,{T}} :=  \hat{Y}^U \text{ where } \max (\hat{y}_{i,k})_{k\in\{1,..., K\} } > \mathcal{T}.
\end{equation}
Empirically,  $ \mathcal{T}$ is set to $0.5$ for all experiments. \cref{sec:stable regularization} additionally proposes the constraints that ensure the quality of the pseudo labels. The logits associated with $\hat{Y}^U_{T}$ are represented by ${Z}^U_{T}$. 
Substitute $Y^U_{T}$ into $I(Y^U; Z^U|X^U)$, we have:
\begin{equation}
\label{eq:our_I_extend}
    \begin{split}
        I(Y^U; Z^U|X^U) \ge & I(Y^U_{T}; Z^U|X^U)         
        :=                  H(Z^U|X^U) - H(Z^U_{T}|X^U).
    \end{split}
\end{equation}
The inequality is maintained because $Y^U_{T}$ constitutes a subset of $Y^U$. 
\cref{eq:our_I_extend} shows that maximizing $I(Y^U_{T}; Z^U|X^U)$ equates to raise a lower bound of $I(Y^U;$ $Z^U|X^U)$. Maximizing \cref{eq:our_I_extend} inherently promotes the independence between known and unknown classes through a strategy that minimizes the entropy of unlabeled logits with certain pseudo labels and maximizes the entropy of the whole unlabeled logits.
It also ensures a uniform distribution of the unlabeled logits across classes until they are assigned reliable pseudo labels.  

Substituting \cref{eq:our_I_extend} into \cref{eq:our_I} (see derivation details in \cref{app:Mathematical details}), we have:
\begin{equation}
\label{eq:our_H}
\begin{split}
    \max_{\theta} &  \underbrace{
    I(X; Z) +  I(Y^L; Z^L|X^L))
     +  H(Z^U|X^U)}_{\text{Introduced by PIM  
     }}\underbrace{ -  H(Z^U_{T}|X^U)}_{\text{Our proposed $\mathcal{L}_R$ }}.
\end{split}
\end{equation}
Finally, the corresponding loss for $- H(Z^U_{T}|X^U)$ is proposed as follows:
\begin{equation}
\label{eq:{L}_{R}}
        \mathcal{L}_{R} (\theta) =  - \frac{\gamma  }{|X^U_{T}|} \sum\nolimits_{x_i\in X^U_{T}} \sum\nolimits_{k=1}^K z_{i,k} \log z_{i,k},
\end{equation}
where $\gamma \in (0,1]$ controls the weight of the extra introduced $\mathcal{L}_R$. Interestingly, $\gamma$ can be connected to various previous works, which will be detailed in \cref{sec:conn}.


\subsection{Integrating with previous approaches}
\label{sec:conn}
This part expounds on the method of integrating the proposed independence constraints with previous work under the the uniform assumption within unconfident predictions. Specifically,  it details the beneficial effects of the additional constraint on independence to the task. 



\textbf{Previous approaches.}
Current InfoMax methods for generalized category discovery tend to maximize the mutual information between $Z$ and $X$ while containing the conditional probability $z_i$ of labeled samples:
\begin{equation}
    \label{eq:old_I}
        \max_{\theta} I(Z;X)
        \;   \; \text{s.t.}
        \; \arg \max (z_i) =\arg \max  (y_i) , \; \; \forall x_i \in X^L.
\end{equation}
Our analysis primarily focuses on the SOTA method, Parametric InfoMax (PIM)~\cite{chiaroni2023parametric} as an example. We express the objective of PIM in the entropy form:
\begin{equation}
    \label{eq:old_H}
    \max_{\theta} { - H_c({Z}^L, Y^L) + H({Z})   -  \lambda \cdot H({Z}^U | X^U)},
\end{equation}
where $\lambda \in (0,1]$ is searched from a finite set on the dataset through a bi-level optimization to control the weight of $- H({Z}^U | X^U)$.
Accordingly, its empirical loss is written as follows:
\begin{equation}
    \label{eq:loss_old_I}
    \begin{split}
        \mathcal{L}_{pim}(\theta)  = & - \frac{1}{|X^L|} \sum\nolimits_{x_i\in X^L} \sum\nolimits_{k=1}^K z_{i,k} \log y_{i,k}                                                  \\
                            & + \sum\nolimits_{k=1}^K \pi_k \log \pi_k - \frac{\lambda}{|X^U|}  \sum\nolimits_{x_i\in X^U}\sum\nolimits_{k=1}^K  z_{i,k} \log z_{i,k},
    \end{split}
\end{equation}
where $\pi_k = P(\arg \max(Y)=k; \theta)$ denotes the marginal distributions.

In comparison to the standard supervised classification approach
(i.e., maximizing $  - H_c({Z}^L, Y^L)$ that is equivalent to $I({Z}^L; Y^L)$~\cite{boudiaf2020unifying}), 
PIM additionally introduced terms ($H({Z}) - \lambda \cdot H({Z}^U | X^U)$), which enjoys an intuitive meaning: encourage predictions with lower confidence to follow the uniform distribution. Specifically, maximizing $ H({Z})$ forces that all samples are evenly assigned to a class while maximizing $ - \lambda \cdot H({Z}^U | X^U)$ encourages high confidence of the predictions. 
However, maximizing $ - \lambda \cdot H({Z}^U | X^U)$ excels the bias of balanced partitions that may be introduced with maximizing $ H({Z})$.
\cref{eq:our_H} reveals that even if the uniform assumption is applied to the unconfident prediction, it still requires further constraints on the independence between known and unknown classes. 
As illustrated in \cref{fig:vis_dist}~(b), without constrained independence, though samples in each cluster are compact, the margin between known and unknown classes is less evident, causing the confusion issue. 
Our results show that integrating the proposed $\mathcal{L}_{R}$ would ameliorate the confusion issue and promote performance, especially on unknown classes.

\textbf{Simultaneously achieving independence and uniform assumption.} 
We further reformulate RPIM's learning objective \cref{eq:our_H} to ensure its alignment with \cref{eq:old_H} by integrating them together:
\begin{equation}
\label{eq:our_H_conn}
\begin{split}
    \max_{\theta} & \; \cref{eq:old_H}  + \gamma \cdot( H(Z^U|X^U) { - H(Z^U_{T}|X^U)}) \\
     \Rightarrow
     \max_{\theta} &  {-H_c(Z^L, Y^L) + H(Z) - \lambda \cdot ( (1-\gamma )  H(Z^U|X^U)} + \frac{\gamma}{\lambda}  {H(Z^U_{T}|X^U)}).
\end{split}
\end{equation}
Empirically, 
we find that $\gamma$ 
is relatively small, hence  $\lambda (1 - \gamma )\approx \lambda$. Please refer to \cref{app:Mathematical details} for derivation details.  
For further simplification, \cref{eq:our_H_conn} can be  approximated as:
\begin{equation}
\label{eq:our_H_final_conn}
     \max_{\theta}  \underbrace{-H_c(Z^L, Y^L) + H(Z) - \lambda \cdot H(Z^U|X^U)}_{\text{Achieved by $\mathcal{L}_{pim}$}}    \underbrace{ -\lambda \cdot \eta \cdot H(Z^U_{T}|X^U)}_{\text{Achieved by $\mathcal{L}_{R}$}}.
\end{equation}
Benefiting from the searched $\lambda$ in PIM, 
we treat $ \eta$ 
as the hyper-parameter and use $\gamma = \lambda \cdot \eta$ for the weight for $\mathcal{L}_{R}$.
We use $\eta = 0.03$ for all experiments. {Our sensitive analysis also shows that the method is not sensitive to the value of $\eta$. We discuss further advantages brought by the integration as follows.

\textbf{Using $\mathcal{L}_{R}$ with $\mathcal{L}_{pim}$ synergetically certifies better results.}
Intuitively,  $\mathcal{L}_{R}$ increase reliability of unlabeled samples, thus  
\cref{eq:our_H_final_conn} further maximizes the mutual information between unlabeled logits and their certain pseudo labels, promoting the independence between known and unknown classes. 
As our extensive experiments validated,  
this combination
promotes the independence between known and unknown classes, improving overall performance. 
Experimental results also depict that, 
without a reliable $\hat{Y}$,
$\mathcal{L}_{R}$ may compromise the performance of the labeled set.

\textbf{Additionally using $\mathcal{L}_{R}$ further reduces empirical risks.}
Incorporating $\hat{Y}$, the overall risk  for the generalized category discovery problem is written as:
\begin{equation}
\label{eq:risk}
    \begin{split}
        R^{all}:= H_c(Z^L, Y^L) + H_c(Z^U, Y^U) - H(Z) + H(Z|X) + R(\hat{Y}^U), 
    \end{split}
\end{equation}
where $R(\hat{Y}^U)$  denotes the risk introduced by $\hat{Y}^U$. 
Note that the positive coefficients of terms are omitted here. 
It can be seen that using \cref{eq:our_H_final_conn} as the objective will minimize all terms except $R(\hat{Y}^U)$ in $R^{all}$.
We show that using \cref{eq:our_H_final_conn} as the objective leads to a lower supremum of the risk than  \cref{eq:old_H}:
\begin{proposition}
\label{prop:our_better_old}
    Given that $\hat{Y}$  and $\hat{Y}^U_{T}$ are reliable and confident, $R(\hat{Y}^U)$ can be considered constant and therefore omitted. Under this assumption, maximizing
    \cref{eq:our_H_final_conn} leads to a lower supremum of the risk than maximizing \cref{eq:old_H}.
\end{proposition}
\begin{proof}
    Comparing two equations, it can find that:
    \begin{equation}
        \cref{eq:our_H_final_conn} + ( H(Z^U_{T}|X^U) - H(Z^U|X^U)) = \cref{eq:old_H}. 
    \end{equation}
    Since $Z^U_{T}$ is a subset of $Z^U$: $Z^U_{T} \subset Z^U$, it is straightforward that $ H(Z^U_{T}|X^U) \leq  H(Z^U|X^U)$ and $ H(Z^U_{T}|X^U) - H(Z^U|X^U) \leq 0$.
    Therefore, using \cref{eq:old_H} as the objective would minimize fewer terms in $R^{all}$ than \cref{eq:our_H_final_conn}. As such,  \cref{eq:our_H_final_conn} leads to a lower supremum of the risk than \cref{eq:old_H}. See more proof details in \cref{app:Mathematical details}.
    \qed
\end{proof}
{\cref{prop:our_better_old} highlights that the proposed objective \cref{eq:our_H_final_conn} will lead to better results than the previous work using \cref{eq:old_H}.}
{However, it is noted that $R(\hat{Y}^U)$ remains when solely using $\mathcal{L}_{R}$. Therefore, we propose a regularization on pseudo labels to further reduce $R(\hat{Y}^U)$ in \cref{eq:risk} and tights \cref{eq:our_H_final_conn}. 
}

\subsection{Ensuring reliable pseudo labels}
\label{sec:stable regularization}


As pseudo labels $\hat{Y}$ are introduced in 
\cref{eq:pesudo_y,eq:our_I_extend}, 
a constraint should be applied to reduce $R(\hat{Y}^U)$.
Specifically, the constraint should  
ensure that $\hat{Y}$ is reliable for $\mathcal{L}_{R}$
during the optimization to mitigate the possible side effects of pseudo labels. 
Since $Y^U$ for $\hat{Y}^U$ cannot be accessed, we take advantage of $Y^L$ and $\hat{Y}^L$ and show that maximizing $ I(\hat{Y}^L, {Y}^L; Z^L|X^L)$ enforces an overall better quality of $\hat{Y}$.

\begin{proposition}
\label{eq:why_Ls}
    Assume the independence assumption $\hat{Y}^L \indep \hat{Y}^U$ holds. 
    Maximizing $ I(\hat{Y}^L,$ $ {Y}^L; Z^L|X^L)$ is equivalent to maximizing $I(\hat{Y}, {Y}^L; Z|X)$.
\end{proposition}

\begin{proof}
Since $\hat{Y}^L \indep \hat{Y}^U$, we have:
\begin{equation}
\label{eq:why}
    \begin{split}
        I(\hat{Y}, {Y}^L; Z|X) \ge
        I(\hat{Y}^U; Z^U|X^U) + I(\hat{Y}^L, {Y}^L ; Z^L|X^L).
    \end{split}
\end{equation}
Since $I(\hat{Y}^U; Z^U|X^U)$ $\ge 0$, it has $ I(\hat{Y}, {Y}^L; Z|X) \ge I(\hat{Y}^L, {Y}^L ; Z^L|X^L)$.
The derivation details of \cref{eq:why} can be seen in \cref{app:Mathematical details}. 
\qed
\end{proof}

{
Consequently, maximizing $I(\hat{Y}, {Y}^L; Z|X)$ implies maximizing the mutual information between pseudo labels and ground-truth labels, which would lead to overall better $\hat{Y}$ so does $\hat{Y}^{U}$. 
Since maximizing mutual information can be changed as minimizing cross-entropy~\cite{boudiaf2020unifying}, we replace $I(\hat{Y}^L, {Y}^L; Z^L|X^L)$ with $H_c( mix(\hat{Y}^L, {Y}^L), Z^L )$, where $mix(\cdot, \cdot)$ denotes a mixing method. We adopt a weighted combination between $\hat{Y}^L$ and ${Y}^L$: $mix(\hat{Y}^L, {Y}^L) = (1-\beta )\cdot{Y}^L + \beta \cdot \hat{Y}^L$, where $\beta$ controls the mixing strength. 
Note that we set $\beta=0.05$ for all experiments. 
Finally, we have the empirical loss that ensures a stable regularization for $H_c( mix(\hat{Y}^L, {Y}^L), Z^L )$ as:
\begin{equation}
        \mathcal{L}_{S} (\theta) =  - \frac{1 }{|X^L|} \sum\nolimits_{x_i\in X^L} \sum\nolimits_{k=1}^K z_{i,k} \log ((1-\beta) y_{i,k} + \beta \hat{y}_{i,k}).
\end{equation}
Intuitively, $\mathcal{L}_{S}$ implicitly aligns $\hat{Y}$ with the anchor $Y^L$, thus promoting the accuracy of known classes, 
yielding better convergence for the optimization process. 
Our experiments validate that using $\mathcal{L}_{S}$ with $\mathcal{L}_{R}$  can lead to further improvements. Importantly, $\mathcal{L}_{S}$ alleviates the possible performance decline caused by $\mathcal{L}_{R}$ of unlabeled samples in known classes as well as resulting in consistent improvements across various datasets.


\subsection{Semantic-bias transformation for latent features refining}
\label{sec:Semantic-bias transformation}

To align with the PIM method and save computational consumption, we employ the latent features outputted from the pre-trained model without directly fine-tuning it. 
Similar to the previous work~\cite{TschannenDRGL20}, our experiments also show that the quality of the latent features $X$ limits the model's performance, and refining $X$ can benefit the downstream generalized category discovery task.
For better notations, we denote $f_\theta$ as $g \circ h$ where $h: X \to X' \in \mathbb{R}^D$ conducts the semantic refining transformations for $X$ and $g: X' \to Z \in \mathbb{R}^K$ where values of variables in each dimension are in the range $[0,1]$.
$h$ acts as the probabilistic parametric classifier. 
{
We notice that the deep latent features from the pre-trained models, which usually represent semantic embeddings, are linearized, normalized, and located in an Affine space~\cite{bengio2013better,upchurch2017deep,wang2021regularizing,tan2023semantic}. 
Therefore, a transformation that does not violate the aforementioned characteristics and maintains the original semantics is critical.
}

One possible approach is to use a transformation such as one linear layer. However, our experiments in  \cref{sec:Ablation studies and analysis}
indicate that using the linear layer with or without bias will lead to performance degradation.
{This phenomenon may result from the hypothesis that the transformation is not the same across samples and should be conditioned by the inputs. Thus, the transformation should take inputs as conditions rather than the aforementioned common linear transformation.} 


In order to obtain a proper transformation, we propose the semantic-bias transformation. Empirically,  one linear layer $mlp(\cdot)$ that takes $X$ as the input is used to learn the bias: $b = mlp(X)$. By adding $b$ to $X$, the final output is normalized: $h(X) = ({x + b})/{||x + b||_2}$.
It is worth noting that $b$ is initialized as zeros. 
Such an approach enables an affine transformation of the semantics, reducing over-transforming risks.
As seen in \cref{sec:Ablation studies and analysis}, our semantic-bias transformation consistently improves results across diverse datasets by enlarging the margin between all classes compared to other potential transformations.

\section{Experiments}

\subsection{Experimental settings}
\textbf{Competitors.}
We compare our proposed method with existing generalized category discovery methods:
GCD~\cite{vaze2022generalized},
and
PIM~\cite{chiaroni2023parametric}.
In particular, PIM based on information maximization is the current state-of-the-art (SOTA) generalized category discovery method.
Additionally, the traditional machine learning method, k-means~\cite{macqueen1967classification}; three novel category discovery methods:
RankStats+~\cite{han2021autonovel},
UNO+~\cite{fini2021unified},
ORCA~\cite{cao2022openworld}; and several information maximization methods:
RIM~\cite{krause2010discriminative},
and TIM~\cite{boudiaf2020information}
are adapted for generalized category discovery as competitors.
The results of the modified novel category discovery methods are reported in~\cite{vaze2022generalized}, and the modified information maximization methods are reported in~\cite{chiaroni2023parametric}.

\textbf{Datasets.}
Six image datasets are adopted to validate the feasibility of our proposed RPIM compared to other competitors. These datasets include three well-known generic object recognition datasets, CIFAR10~\cite{krizhevsky2009learning},
CIFAR100~\cite{krizhevsky2009learning} 
and 
ImageNet-100~\cite{deng2009imagenet}. Additionally, 
two fine-grained datasets CUB~\cite{wah2011caltech} and Stanford Cars~\cite{krause20133d} that encompass fine-grained categories posited to be more challenging to differentiate than the more generic object classes; as well as the long-tail dataset Herbarium19~\cite{tan2019herbarium} that  
mirroring real-world scenarios with pronounced class imbalances, large intra-class variations, and low inter-class variations
are also incorporated.
These diverse datasets are intended to comprehensively validate our approach's effectiveness compared to other competitors.
Following GCD and
PIM~\cite{vaze2022generalized,chiaroni2023parametric}, the initial training set of each dataset is partitioned into labeled and unlabeled subsets. To elaborate, half of the image samples affiliated with the known classes are allocated to the labeled subset, while the remaining half are assigned to the unlabeled subset. The unlabeled subset also includes all image samples from the remaining classes in the original dataset, designated as novel classes. Consequently, the unlabeled subset comprises instances from $K$ different classes. 

{\textbf{Training details.} Consistent with PIM, we utilize latent features extracted by the feature encoder DINO (VIT-B/16)~\cite{caron2021emerging} that is pre-trained on ImageNet~\cite{deng2009imagenet} through self-supervised learning. Our experiments unveil that the model exhibits sensitivity to hyper-parameters, especially weight decay. To address this, a methodology is devised for weight decay searching, involving the construction of smaller labeled and unlabeled subsets derived solely from the labeled data. 
Specifically, the labeled set is divided into two subsets: a sub-labeled set comprising half of the known classes and a sub-unlabeled set encompassing all known classes, adhering to the generalized category discovery setting. 
For details on the optimized weight decay values, please see \cref{app:More experimental details}. It is worth noticing that our PRIM improves both situations where the searched weight decay values are used or not (see more results and analysis in \cref{app:more results}).}

\textbf{Evaluation metric.}
{Following with prior works \cite{vaze2022generalized,chiaroni2023parametric}, we use the proposed accuracy metric from \cite{vaze2022generalized} of all classes, known classes, and unknown classes for evaluation. Please see a detailed description of the experimental setup in \cref{app:More experimental details}. 
The implementation code will be made publicly available following the acceptance of our manuscript.}


\subsection{Main results}
In this part, we compare RPIM with previous methods across six datasets. The averaged results across all datasets are shown in \cref{fig:all_res_avg}, and the detailed results of each dataset are shown in \cref{tab:all_res_main}.  Please refer to more results and discussions in \cref{app:more results}.

\begin{table}[t]
\centering
\caption{Main results: Accuracy scores 
across fine-grained and generic datasets of our RPIM and other competitors. The best results of each group are highlighted in \textbf{bold}. Please refer to \cref{fig:all_res_avg} for averaged results on all datasets.
}
\label{tab:all_res_main}
\resizebox{1\textwidth}{!}{%
    \begin{tabular}{l|ccc|ccc|ccc}
        \toprule
                                                                        & \multicolumn{3}{c|}{\textbf{CUB}}     & \multicolumn{3}{c|}{\textbf{Stanford Cars}} & \multicolumn{3}{c}{\textbf{Herbarium19}}                                                                                                  \\
        \hline Approach & \multicolumn{1}{c}{{\;\;\;\;\;All\;\;\;\;\;}} & \multicolumn{1}{c}{{\;Known\;}} & \multicolumn{1}{c|}{{\;Unknown\;}} & \multicolumn{1}{c}{{\;\;\;\;\;All\;\;\;\;\;}} & \multicolumn{1}{c}{{\;Known\;}} & \multicolumn{1}{c|}{{\;Unknown\;}} & \multicolumn{1}{c}{{\;\;\;\;\;All\;\;\;\;\;}} & \multicolumn{1}{c}{{\;Known\;}} & \multicolumn{1}{c}{{\;Unknown\;}}           \\
        \hline
        K-means                                                         & 34.3                        & 38.9                              & 32.1                             & 12.8          & 10.6          & 13.8          & 12.9          & 12.9          & 12.8          \\
        RankStats+ \cite{han2021autonovel} (TPAMI-21)                   & 33.3                        & 51.6                              & 24.2                             & 28.3          & 61.8          & 12.1          & 27.9          & 55.8          & 12.8          \\
        UNO+ \cite{fini2021unified} (ICCV-21)                           & 35.1                        & 49.0                              & 28.1                             & 35.5          & \textbf{70.5} & 18.6          & 28.3          & \textbf{53.7}          & 14.7          \\
        ORCA \cite{cao2022openworld} (ICLR-22)                          & 27.5                        & 20.1                              & 31.1                             & 15.9          & 17.1          & 15.3          & 22.9          & 25.9          & 21.3          \\ 
        ORCA \cite{cao2022openworld} - ViTB16                           & 38.0                        & 45.6                              & 31.8                             & 33.8          & 52.5          & 25.1          & 25.0          & 30.6          & 19.8          \\ %
        GCD \cite{vaze2022generalized} (CVPR-22)                        & \textbf{51.3}                        & \textbf{56.6}                              & \textbf{48.7}                             & \textbf{39.0}          & 57.6          & \textbf{29.9}          & \textbf{35.4}          & 51.0          & \textbf{27.0}          \\
        \hline
        \multicolumn{1}{c}{}& \multicolumn{9}{c}{InfoMax based methods} \\ \hline
        RIM \cite{krause2010discriminative} (NeurIPS-10) (semi-sup.)    & 52.3                        & 51.8                              & 52.5                             & 38.9          & 57.3          & 30.1          & 40.1          & \textbf{57.6} & 30.7          \\
        TIM \cite{boudiaf2020information} (NeurIPS-20)                  & 53.4                        & 51.8                              & 54.2                             & 39.3          & 56.8          & 30.8          & 40.1          & 57.4          & 30.7          \\ \hline
        PIM~\cite{chiaroni2023parametric} (ICCV-23)                                                           & 62.7                        & 75.7                              & 56.2                             & 43.1          & 66.9          & 31.6          & 42.3          & 56.1          & 34.8          \\
        \rowcolor{mygray}\textbf{RPIM (Ours)}                                                            & \textbf{66.8}               & \textbf{77.3}                    & \textbf{61.5}                    & \textbf{45.8} & \textbf{67.5}          & \textbf{35.3}  & \textbf{43.0} & 57.4          & \textbf{35.2} \\ 
        \toprule
        \toprule
                                                                        & \multicolumn{3}{c|}{\textbf{CIFAR10}} & \multicolumn{3}{c|}{\textbf{CIFAR100}}      & \multicolumn{3}{c}{\textbf{ImageNet-100}}                                                                                                 \\
        \hline Approach & \multicolumn{1}{c}{{All}} & \multicolumn{1}{c}{{Known}} & \multicolumn{1}{c|}{{Unknown}} & \multicolumn{1}{c}{{All}} & \multicolumn{1}{c}{{Known}} & \multicolumn{1}{c|}{{Unknown}} & \multicolumn{1}{c}{{All}} & \multicolumn{1}{c}{{Known}} & \multicolumn{1}{c}{{Unknown}}           \\
        \hline
        K-means                                                         & 83.6                        & 85.7                              & 82.5                             & 52.0          & 52.2          & 50.8          & 72.7          & 75.5          & 71.3          \\
        RankStats+ \cite{han2021autonovel} (TPAMI-21)                   & 46.8                        & 19.2                              & 60.5                             & 58.2          & \textbf{77.6}          & 19.3          & 37.1          & 61.6          & 24.8          \\
        UNO+ \cite{fini2021unified} (ICCV-21)                           & 68.6                        & \textbf{98.3}                     & 53.8                             & 69.5          & 80.6          & 47.2          & 70.3          & \textbf{95.0}          & 57.9          \\
        ORCA \cite{cao2022openworld} (ICLR-22)                          & 88.9                        & 88.2                              & 89.2                             & 55.1          & 65.5          & 34.4          & 67.6          & 90.9          & 56.0          \\ 
        ORCA \cite{cao2022openworld} - ViTB16                           & \textbf{97.1}               & 96.2                              & \textbf{97.6}                    & 69.6          & {76.4}          & 56.1          & \textbf{76.5}          & {92.2}          & \textbf{68.9}          \\ %
        GCD \cite{vaze2022generalized} (CVPR-22)                        & 91.5                        & 97.9                              & 88.2                             & \textbf{70.8}          & \textbf{77.6}          & \textbf{57.0}          & 74.1          & 89.8          & 66.3          \\
        \hline
        \multicolumn{1}{c}{} & \multicolumn{9}{c}{InfoMax based methods} \\ \hline
        RIM \cite{krause2010discriminative} (NeurIPS-10) (semi-sup.)    & 92.4                        & \textbf{98.1}                              & 89.5                             & 73.8          & 78.9          & 63.4          & 74.4          & 91.2          & 66.0          \\
        TIM \cite{boudiaf2020information} (NeurIPS-20)                  & 93.1                        & 98.0                              & 90.6                             & 73.4          & 78.3          & 63.4          & 76.7          & 93.1          & 68.4          \\
        \hline
        PIM~\cite{chiaroni2023parametric} (ICCV-23)                                                          & 94.7                        & 97.4                              & 93.3                             & 78.3          & 84.2          & 66.5          & 83.1          & \textbf{95.3} & 77.0          \\
        \rowcolor{mygray} \textbf{RPIM (Ours)}                                                            & \textbf{95.3}                       & {97.6}                              & \textbf{94.1}                              & \textbf{82.7}  & \textbf{85.4}  & \textbf{77.2}  & \textbf{83.2}  & 95.2          & \textbf{77.2} \\ 
        \bottomrule
    \end{tabular}
}
\end{table}
\begin{figure}[!t]
  \centering
  \includegraphics[width=0.8\linewidth]{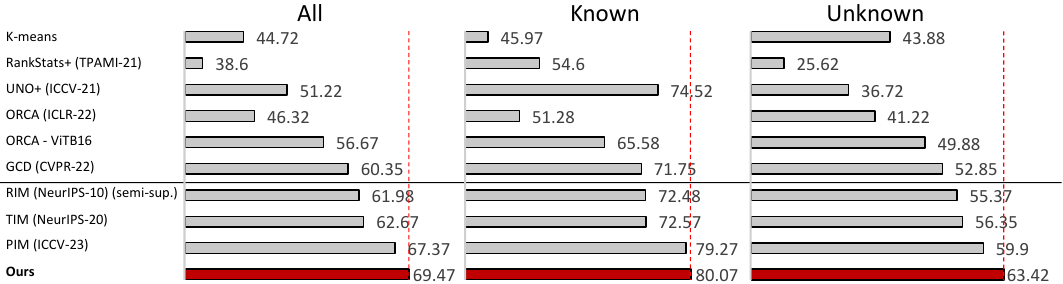}
  \caption{Averaged results across all datasets of 
    k-means~\cite{macqueen1967classification},
    RankStats+~\cite{han2021autonovel},
    UNO+~\cite{fini2021unified},
    ORCA~\cite{cao2022openworld},
    GCD~\cite{vaze2022generalized},
    RIM~\cite{krause2010discriminative},
    TIM~\cite{boudiaf2020information},
    PIM~\cite{chiaroni2023parametric},
    and our proposed RPIM of all classes, known classes, and unknown classes.  
  }
  \label{fig:all_res_avg}
\end{figure}


\textbf{Comparison to Previous Methods.} As illustrated in \cref{fig:all_res_avg}, RPIM outperforms all prior methods in terms of average accuracy for both known and unknown classes. Detailed comparisons in \cref{tab:all_res_main} reveal that RPIM sets new SOTA results on five out of six datasets. Unlike other InfoMax-based methods, our approach consistently enhances performance across all datasets for unknown classes, while also improving the accuracy of known classes in five datasets. 
These results highlight the effectiveness of RPIM in alleviating the confusion problem inherent in earlier models.


\textbf{Comparison to the Current SOTA Method.} RPIM achieves a notable improvement over the current SOTA method, PIM, as evidenced by average accuracy gains exhibited in \cref{fig:all_res_avg}: a $3.52\%$ increase for unknown classes and a $0.8\%$ rise for known classes, cumulating in an overall enhancement of $2.1\%$ across all classes. As demonstrated in \cref{tab:all_res_main}, significant advancements are observed with increases of $10.7\%$, $5.3\%$, and $3.7\%$ for unknown classes on CIFAR100, CUB, and Stanford Cars datasets, respectively. Furthermore, RPIM surpasses PIM across all unknown classes and the majority of known classes throughout all datasets, evidencing \cref{prop:our_better_old}.

\textbf{Remarks of our proposed method.} 
An observation from the result, our method markedly enhances the model's ability to discover unknown classes.
This improvement 
does not detract from, but enhances the model's performance on known classes under most circumstances,
underscoring the balanced generalization ability of RPIM. 
Such outcomes support our argument that the independence between known and unknown classes should be constrained.

\subsection{Ablation studies and analysis}
\label{sec:Ablation studies and analysis}

This section delves into the ablation studies to dissect the contributions of individual components within RPIM. 
The quantitative ablation results are presented in \cref{tab:ablation_res} while  \cref{fig:ablation_avg} reports the average results of ablation studies.
For a comprehensive sensitivity and hyper-parameter analysis, please refer to details in \cref{app:more results}.

\textbf{Semantic-bias transformation $h$ promotes probabilistic parametric classifier.} 
It can be observed that using $h$ consistently brings improvements, especially on unknown classes. However, the improvement on both known and unknown classes
is not certified without the proposed $\mathcal{L}_{R}$ and $\mathcal{L}_{S}$.

\textbf{$\mathcal{L}_{R}$ promotes performance on the unknown classes.} \cref{tab:ablation_res}, \cref{tab:ablation_res_use_fixed_wd} and \cref{fig:ablation_avg}  show that additionally using $\mathcal{L}_{R}$ can boost performance on unknown classes across different settings and datasets. However, without using $\mathcal{L}_{S}$, $\mathcal{L}_{R}$, the model may compromise the performance of known classes. 
Overall, $\mathcal{L}_{R}$ leads to improvements in all classes due to its significant promotions in unknown classes in comparison to its compromised performance of known classes, thus verifying our analysis in \cref{prop:our_better_old}.  

\textbf{$\mathcal{L}_{S}$ alleviates compromises in known classes brought by $\mathcal{L}_{R}$ and even improve the performance.} \cref{tab:ablation_res,tab:ablation_res_use_fixed_wd,fig:ablation_avg} show all results indicate that using  $\mathcal{L}_{S}$ with  $\mathcal{L}_{R}$ ensuring better convergence than solely using  $\mathcal{L}_{S}$. Specially, $\mathcal{L}_{S}$ alleviates the side-effect of $\mathcal{L}_{R}$  on known classes. This implies that $\mathcal{L}_{S}$ produces reliable pseudo labels
validating \cref{eq:why_Ls}.

\begin{table}[t]
\caption{Ablation results: Accuracy scores across fine-grained and generic datasets of each setting. 
The best results are highlighted in \textbf{bold}.
Improvement and degradation in our approach from the baseline are highlighted in \red{red$\uparrow$} and \blue{blue$\downarrow$}, respectively.
$h$ denotes the proposed semantic-bias transformation. 
}
\label{tab:ablation_res}
\resizebox{\textwidth}{!}{%
\begin{tabular}{llccc|ccc|ccc|ccc}
\toprule
&& \multicolumn{3}{c|}{}  & \multicolumn{3}{c|}{\textbf{CUB}}     & \multicolumn{3}{c|}{\textbf{Stanford Cars}} & \multicolumn{3}{c}{\textbf{Herbarium19}}                                                                  \\   \hline
 ID & Settings & \textbf{$h$} & \textbf{$\mathcal{L}_{R}$} & \textbf{$\mathcal{L}_{S}$} & \multicolumn{1}{c}{{\;\;\;\;\;\;All\;\;\;\;\;\;}} & \multicolumn{1}{c}{{Known}} & \multicolumn{1}{c|}{{Unknown}} & \multicolumn{1}{c}{{\;\;\;\;\;\;All\;\;\;\;\;\;}} & \multicolumn{1}{c}{{Known}} & \multicolumn{1}{c|}{{Unknown}} & \multicolumn{1}{c}{{\;\;\;\;\;\;All\;\;\;\;\;\;}} & \multicolumn{1}{c}{{Known}} & \multicolumn{1}{c}{{Unknown}} \\ \hline
1 & Baseline ($\mathcal{L}_{pim}$)&                 &                         &                 & 62.7                             & 75.7                             & 56.2                             & 43.1                             & 66.9                             & 31.6                             & 42.3                             & 56.1                             & 34.8                             \\ \hline
2 & Baseline tuned ($\mathcal{L}_{pim}$)&                 &                         &                 & 64.8                             & 75.1                             & 59.6                             & 42.6                             & 59.3                              & 34.6                              & 43.1                             & 57.6                             & 35.4                             \\ 
3 &$\mathcal{L}_{pim}$+$\mathcal{L}_{R}$ &                 & \checkmark                         &                 & 66.2                             & 75.1                             & 61.8                             & 42.3                             & 57.6                             & 34.9                             & 43.1                             & 56.7                             & 35.8                             \\
4 & $\mathcal{L}_{pim}$+$\mathcal{L}_{S}$ &                 &                         & \checkmark                 & 64.5                             & 74.5                             & 59.4                             & 43.8                             & 59.6                             & 36.2                             & 42.8                             & \textbf{58.0}                               & 34.6                             \\
5 & $\mathcal{L}_{pim}$+$\mathcal{L}_{R}$+$\mathcal{L}_{S}$ &                 & \checkmark                         & \checkmark                 & 66.3                             & 76.7                             & 61.1                             & 43.7                             & 60.3                             & 35.7                             & 42.8                             & \textbf{58.0}                               & 34.6                             \\
\hline
\multicolumn{4}{c}{} & \multicolumn{9}{c}{Using semantic transformation} \\ \hline
6 & $\mathcal{L}_{pim}$+$h$ & \checkmark                 &                         &                 & 64.9                             & 76.7                             & 58.9                             & 44.7                             & 65.8                             & 34.6                             & 43.0                               & 57.4                             & 35.2                             \\
7 & $\mathcal{L}_{pim}$+$h$+$\mathcal{L}_{R}$ & \checkmark                 & \checkmark                         &                 & 66.3                             & 76.2                             & 61.3                             & 44.4                             & 65.4                             & 34.3                             & 43.0                               & 56.6                             & \textbf{35.7}                             \\
8 & $\mathcal{L}_{pim}$+$h$+$\mathcal{L}_{S}$ & \checkmark                 &                         & \checkmark                 & 64.9                             & 75.5                             & 59.7                             & \textbf{45.8}                             & 66.7                             & \textbf{35.7}                             & \textbf{43.2}                             & 57.4                             & 35.6                             \\
\rowcolor{mygray} 9 &  \textbf{Ours} ($\mathcal{L}_{pim}$+$h$+$\mathcal{L}_{R}$+$\mathcal{L}_{S}$) & \textbf{\checkmark}        & \textbf{\checkmark}                & \textbf{\checkmark}        & \textbf{66.8}~\red{($4.1\uparrow$)}               & \textbf{77.3}~\red{($1.6\uparrow$)}                    & \textbf{61.5}~\red{($5.3\uparrow$)}                  & \textbf{45.8}~\red{($2.7\uparrow$)} & \textbf{67.5}~\red{($0.6\uparrow$)}          & {35.3}~\red{($3.7\uparrow$)}  & {43.0}~\red{($0.7\uparrow$)} & 57.4~\red{($1.3\uparrow$)}          & {35.2}~\red{($0.6\uparrow$)} \\ 
                        \toprule
                        \toprule
& & \multicolumn{3}{c|}{}           & \multicolumn{3}{c|}{\textbf{CIFAR10}} & \multicolumn{3}{c|}{\textbf{CIFAR100}}      & \multicolumn{3}{c}{\textbf{ImageNet-100}}                                                           \\
\hline
ID &  Settings  & \textbf{$h$} & \textbf{$\mathcal{L}_{R}$} & \textbf{$\mathcal{L}_{S}$} & \multicolumn{1}{c}{{All}} & \multicolumn{1}{c}{{Known}} & \multicolumn{1}{c|}{{Unknown}} & \multicolumn{1}{c}{{All}} & \multicolumn{1}{c}{{Known}} & \multicolumn{1}{c|}{{Unknown}} & \multicolumn{1}{c}{{All}} & \multicolumn{1}{c}{{Known}} & \multicolumn{1}{c}{{Unknown}} \\
  \hline
1 & Baseline ($\mathcal{L}_{pim}$)&                 &                         &                 & 94.7                             & 97.4                             & 93.3                             & 78.3                             & 84.2                             & 66.5                             & 83.1                             & \textbf{95.3}                             & 77.0                               \\ \hline
2 & Baseline tuned ($\mathcal{L}_{pim}$)&                 &                         &                 & 95.0                               & 96.1                             & 94.4                             & 80.3                             & 84.6                             & 71.8                             & 83.5                             & 95.0                               & 77.7                             \\ 
3 &$\mathcal{L}_{pim}$+$\mathcal{L}_{R}$ &                 & \checkmark                         &                 & 94.9                             & 96.0                               & 94.4                             & 80.0                               & 83.2                             & 73.6                             & 83.5                             & 95.0                               & 77.7                             \\
4 & $\mathcal{L}_{pim}$+$\mathcal{L}_{S}$ &                 &                         & \checkmark                 & 94.9                             & 97.4                             & 93.7                             & 81.4                             & 85.7                             & 72.9                             & 83.6                             & 95.0                               & 77.9                             \\
5 & $\mathcal{L}_{pim}$+$\mathcal{L}_{R}$+$\mathcal{L}_{S}$  &                 & \checkmark                         & \checkmark                 & 94.9                             & 97.4                             & 93.6                             & 81.4                             & 85.7                             & 72.9                             & 83.6                             & 95.0                               & 77.9                             
\\
\hline
\multicolumn{4}{c}{} & \multicolumn{9}{c}{Using semantic transformation} \\ \hline
6 & $\mathcal{L}_{pim}$+$h$ & \checkmark                 &                         &                 & 94.7                             & 97.5                             & 93.3                             & 80.8                             & 84.6                             & 73.1                             & 83.1                             & 95.0                               & 77.1                             \\
7 & $\mathcal{L}_{pim}$+$h$+$\mathcal{L}_{R}$ & \checkmark                 & \checkmark                         &                 & 94.7                             & 97.5                             & 93.2                             & 80.0                               & 83.2                             & 73.6                             & 83.1                             & 95.0                               & 77.1                             \\
8 & $\mathcal{L}_{pim}$+$h$+$\mathcal{L}_{S}$& \checkmark                 &                         & \checkmark                 & 95.2                             & \textbf{97.7}                             & 94.0                               & 82.6                             & \textbf{85.4}                             & 76.8                             & 83.2                             & {95.2}                             & 77.2                             \\
\rowcolor{mygray} 9 & \textbf{Ours} ($\mathcal{L}_{pim}$+$h$+$\mathcal{L}_{R}$+$\mathcal{L}_{S}$) & \textbf{\checkmark}        & \textbf{\checkmark}                & \textbf{\checkmark}        & \textbf{95.3}~\red{($0.6\uparrow$)}                        & 97.6~\red{($0.2\uparrow$)}                              & \textbf{94.1}~\red{($0.8\uparrow$)}                              & \textbf{82.7}~\red{($4.4\uparrow$)}  & \textbf{85.4}~\red{($1.2\uparrow$)}  & \textbf{77.2}~\red{($10.7\uparrow$)}  & \textbf{83.2}~\red{($0.1\uparrow$)}  & 95.2~\blue{($0.1\downarrow$)}          & \textbf{77.2}~\red{($0.2\uparrow$)} \\ 
                        \bottomrule
\end{tabular}%
}
\end{table}

\begin{figure}[!t]
\vspace{-3pt}
    \centering
    \includegraphics[width=0.9\linewidth]{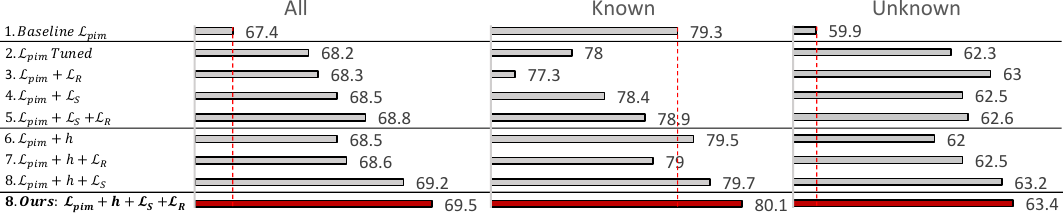}
    \caption{Averaged results across all datasets of ablation studies and our proposed PRIM. 
    }
    \label{fig:ablation_avg}
\vspace{-9pt}
\end{figure}

\textbf{Overall effectiveness of all components.}
Integrating all proposed components leads to optimal performance, chiefly by enhancing the separability between known and unknown classes through synergistic effects. Specifically, the class assigned to pink exhibits a distinct separation from that marked in red when all components are employed, as opposed to scenarios lacking this integration as shown in \cref{fig:vis_dist}. This is quantitatively reflected in our performance metrics, where an average enhancement of $0.8\%$ for known classes and $3.5\%$ for unknown classes across all evaluated datasets to baseline. The overall improvements across all datasets also suggest the effectiveness of our approach.

\textbf{Semantic-bias transformations analysis.}
{Compared to tuning the whole pre-trained model, such as DINO (VIT-B/16) with 85M parameters, our introduced semantic-bias transformations with only 0.26M parameters significantly save the computational cost.}
A comprehensive comparison of possible transformations across various datasets reveals our approach's superiority in ensuring consistent performance enhancements, as documented in \cref{tab:semantic}. On the contrary, alternative strategies fail to certify such consistency. For instance, the employment of a linear layer without bias significantly improves accuracy on the ImageNet-100 dataset but degrades performance on the CIFAR10 and CIFAR100 datasets. \cref{fig:transformations} further elucidates that the default linear layers, whether biased or not, often fail to establish clear decision boundaries between classes. In contrast, our method consistently facilitates improvements. Please see the visualization of latent features with different transformations in \cref{app:more results}.

\begin{table}[t]
\caption{Semantic transformation: Comparison between different transformations. 
All results better than the baseline are highlighted in \textbf{bold}. 
}
\label{tab:semantic}
\resizebox{\textwidth}{!}{%
\begin{tabular}{l|ccc|ccc|ccc}
\toprule
 & \multicolumn{3}{c|}{\textbf{CUB}} & \multicolumn{3}{c|}{\textbf{Stanford Cars}} & \multicolumn{3}{c}{\textbf{Herbarium19}} \\ \hline
Setting & \multicolumn{1}{c}{{\;\;\;\;All\;\;\;\;}} & \multicolumn{1}{c}{{Known}} & \multicolumn{1}{c|}{{Unknown}} & \multicolumn{1}{c}{{\;\;\;\;All\;\;\;\;}} & \multicolumn{1}{c}{{Known}} & \multicolumn{1}{c|}{{Unknown}} & \multicolumn{1}{c}{{\;\;\;\;All\;\;\;\;}} & \multicolumn{1}{c}{{Known}} & \multicolumn{1}{c}{{Unknown}} \\ \hline
Baseline & 62.7 & 75.7 & 56.2 & 43.1 & 66.9 & 31.6 & 42.3 & 56.1 & 34.8 \\ \hline
Linear layer without bias & 59.3 & 74.1 & 51.9 & 39.4 & 63.8 & 27.6 & 41.2 & 53.9 & 34.3 \\
Linear layer with bias & 52.6 & 67.2 & 45.3 & \textbf{45.6} & 66.4 & \textbf{35.6} & 40.9 & 55.6 & 33.0 \\
Learned input conditioned weight and bias & \textbf{64.1} & 73.8 & \textbf{59.3} & 39.0 & 64.6 & 26.6 & \textbf{42.4} & \textbf{56.5} & \textbf{34.9} \\
\rowcolor{mygray}\textbf{Ours}: Learned input conditioned  bias only & \textbf{66.8} & \textbf{77.3} & \textbf{61.5} & \textbf{45.8} & \textbf{67.5} & \textbf{35.3} & \textbf{43.0} & \textbf{57.4} & \textbf{35.2} \\ \hline
 & \multicolumn{3}{c|}{\textbf{CIFAR10}} & \multicolumn{3}{c|}{\textbf{CIFAR100}} & \multicolumn{3}{c}{\textbf{ImageNet-100}} \\ \hline
Setting & \multicolumn{1}{c}{{All}} & \multicolumn{1}{c}{{Known}} & \multicolumn{1}{c|}{{Unknown}} & \multicolumn{1}{c}{{All}} & \multicolumn{1}{c}{{Known}} & \multicolumn{1}{c|}{{Unknown}} & \multicolumn{1}{c}{{All}} & \multicolumn{1}{c}{{Known}} & \multicolumn{1}{c}{{Unknown}} \\ \hline
Baseline & 94.7 & 97.4 & {93.3} & 78.3 & 84.2 & 66.5 & 83.1 & {95.3} & 77 \\ \hline
Linear layer without bias & 60.8 & 58.8 & 61.8 & \textbf{81.1} & \textbf{84.7} & \textbf{73.8} & \textbf{85.7} & 95.2 & \textbf{80.9} \\
Linear layer with bias & 81.2 & 58.3 & 92.6 & 77.6 & \textbf{84.7} & 63.4 & \textbf{85.1} & 95.2 & \textbf{80} \\
Learned input conditioned weight and bias & 95 & \textbf{97.7} & 93.6 & 81.8 & \textbf{85.4} & 74.5 & 82.5 & 95.2 & 76.1 \\
\rowcolor{mygray}\textbf{Ours}: Learned input conditioned  bias only & \textbf{95.3} & \textbf{97.6} & \textbf{94.1} & \textbf{82.7} & \textbf{85.4} & \textbf{77.2} & \textbf{83.2} & 95.2 & \textbf{77.2} \\
\bottomrule
\end{tabular}%
}
\vspace{-9pt}
\end{table}

\section{Conclusion} 
{Based on the probabilistic parametric classifier under InfoMax, we reveal that applying uniform probability distribution assumption on unconfident predictions is insufficient. Without constraining the independence between known and unknown classes, a confusion issue may emerge and the performance would be compromised.}
Our RPIM alleviates the confusion issue by incorporating pseudo labels as extra supervision and proposes a loss to promote the pseudo labels' quality.
Additionally, we propose a pragmatic semantic-bias transformation to refine semantic features for promoting the probabilistic parametric classifier.
Rigorous theoretical and empirical evaluation indicates that our RPIM establishes new SOTA results.

\textbf{Limitations.} 
Similar to other current existing generalized category discovery methods, our method also faces the limitation that the models require access to the entire target unlabeled dataset at the test. 
However, with a feasible approach that can scale up the $Y$ space, our proposed losses and semantic-aware transformation can be extended in such application scenarios.

\bibliographystyle{splncs04}
\bibliography{main.bib}

\newpage
\appendix


\section{Mathematical details}
\label{app:Mathematical details}

\textbf{More derivations of \cref{eq:our_I_extend}.}
\begin{equation}
\label{eq:our_I_extend_details}
    \begin{split}
        I(Y^U; Z^U|X^U) \ge & I(Y^U_{T}; Z^U|X^U)             \\
        =                   & H(Z^U|X^U) - H(Z^U|X^U,Y^U_{T}) .
    \end{split}
\end{equation}
Due to that $Z^U_{T}$ is selected based on $Y^U_{T}$ so that $H(Z^U|X^U,Y^U_{T}):=H(Z^U_{T}|X^U) $,
\cref{eq:our_I_extend_details} can be further defined as:
\begin{equation}
        H(Z^U|X^U) - H(Z^U,Y^U_{T}|X^U) 
        :=                 H(Z^U|X^U) - H(Z^U_{T}|X^U).
\end{equation}

\textbf{More derivation details of \cref{eq:our_H}.}
\begin{equation}
    \label{eq:our_I_detials}
    \begin{split}
      & \max_{\theta} I(X;Y;Z) \\
    =&  I(X;Z)+   I(Y;Z|X) \\
    = & I(X;Z) +   I(Y^L;Z^L|X)  +   I(Y^U;Z^U|X) \\
     = & I(X;Z) +   I(Y^L;Z^L|X)  +   H(Z^U|X^U) -  H(Z^U_{T}|X^U).
    \end{split}
\end{equation}

\textbf{More derivations of \cref{eq:why}.}
\begin{equation}
\label{eq:why_details}
    \begin{split}
        &I(\hat{Y}^U; Z^U|X^U) + I(\hat{Y}^L; Z^L|X^L) \le I(\hat{Y}; Z|X) \\
         \Rightarrow 
         &I(\hat{Y}^U; Z^U|X^U) + I(\hat{Y}^L; Z^L|X^L) + I({Y}^L; Z^L|X^L) \\ & \le I(\hat{Y}; Z|X) + I({Y}^L; Z^L|X^L)\\
        \Rightarrow
        &I(\hat{Y}^U; Z^U|X^U) + I(\hat{Y}^L, {Y}^L ; Z^L|X^L) \le I(\hat{Y}, {Y}^L; Z|X).
    \end{split}
\end{equation}

\textbf{More derivations of \cref{eq:our_H_conn}}

\begin{equation}
\begin{split}
    \max_{\theta} & \; \cref{eq:old_H}  + \gamma \cdot( H(Z^U|X^U) { - H(Z^U_{T}|X^U)}) \\
     \Rightarrow
     \max_{\theta} &  -H_c(Z^L, Y^L) + H(Z) - \lambda \cdot H(Z^U|X^U) + \gamma  \cdot (H(Z^U|X^U) -  {H(Z^U_{T}|X^U)}), \\
     = 
     \max_{\theta} &  {-H_c(Z^L, Y^L) + H(Z) - \lambda \cdot ( (1-\gamma )  H(Z^U|X^U)} - \frac{\gamma}{\lambda}  {H(Z^U_{T}|X^U)}).
\end{split}
\end{equation}
Empirically, we find that $\eta = \frac{\gamma}{\lambda}$ locates in range $[0.01,  0.05]$ yields acceptable results. 
Since $\lambda$ is in the range $(0, 1]$
it can be observed that $\gamma \ge \eta$ and thus $\lambda(1-\gamma ) \approx \lambda$.

\textbf{Proof  of \cref{prop:our_better_old}.}
\begin{proof}
    Comparing two equations, it can be found that:
    \begin{equation}
        \cref{eq:our_H_final_conn} + \lambda \cdot \gamma \cdot( H(Z^U_{T}|X^U) - H(Z^U|X^U)) = \cref{eq:old_H}, 
    \end{equation}
    Since $Z^U_{T}$ is a subset of $Z^U$ : $Z^U_{T} \subset Z^U$, it is straightforward that: $ H(Z^U_{T}|X^U) \leq  H(Z^U|X^U)$ and $ H(Z^U_{T}|X^U) - H(Z^U|X^U) \leq 0$. 
    Due to
    \begin{equation}
        \lambda, \gamma >0, - \lambda \cdot \gamma \cdot ( H(Z^U_{T}|X^U) 
    - H(Z^U|X^U)) \ge 0.
    \end{equation}
    Therefore, using \cref{eq:old_H} as the objective would minimize fewer terms in $R^{all}$ than \cref{eq:our_H_final_conn}:
    \begin{equation}
    \begin{split}
            &\underbrace{\sup\min[\cref{eq:our_H_final_conn}]}_{\text{supremum of risk when $\min$  \cref{eq:our_H_final_conn} }}   \\
            \ge & \underbrace{\sup \min[ \cref{eq:old_H} ]}_{\text{supremum of risk when $\min$ \cref{eq:old_H} }}.
    \end{split}
    \end{equation}
    In other words,  \cref{eq:our_H_final_conn} leads to a lower supremum of the risk than \cref{eq:old_H}. 
    This completes the proof.
    \qed
\end{proof}

\section{More experimental details}
\label{app:More experimental details}
\textbf{Parameter searching.}
To conduct parameter searching, we split labeled samples and constructed a `smaller' sub-labeled and sub-unlabeled set. Specifically, we take samples under $50\%$ of known classes as the sub-unlabeled samples from unknown classes; furthermore, we take $25\%$ samples from the other $50\%$ known classes as the sub-unlabeled samples from known classes. The left samples are treated as sub-labeled samples. Then,  the hyper-parameters are searched on the sub-labeled and un-labeled sets.

We found that the values of weight decay affect the performance significantly. 
The weight decay value searched from the list $[0.001, 0.002, 0.005, 0.01,$ $0.02, 0.05]$ that has the best performance on the sub-unlabeled set is chosen; due to that, the labeled and unlabeled sets are doubled-size of the subsets, its data manifold is more complex when using the complete dataset. 
Therefore, the chosen weight decay is divided by two for the final training. 
The searched weight decay values are exhibited as \cref{tab:tuned weight decay}.

\begin{table}[ht]
\caption{Tuned weight decay values for each dataset.}
\label{tab:tuned weight decay}
\resizebox{\columnwidth}{!}{%
\begin{tabular}{l|cccccc}
\toprule 
                     & CUB    & Standford Cars & Herbarium19 & CIFAR10 & CIFAR100 & ImageNet-100 \\ \hline
Tuned weighted decay & 0.02/2 & 0.02/2         & 0.02/2      & 0.05/2  & 0.005/2  & 0.005/2     \\ 
\bottomrule
\end{tabular}%
}
\end{table}

\textbf{Statistics of datasets.}
We present the statistics of datasets in \cref{tab:stats}.
Be noted that Herbarium19 is a long-tailed dataset, which
reflects a real-world use case with severe class imbalance
along with large intra-class and low inter-class variations.

\begin{table}[t]
\centering
\caption{Statistics of datasets.}
\label{tab:stats}
\scriptsize
\begin{tabular}{p{1.5cm}|cccccc}
\toprule 
        & CUB  & Standford Cars & Herbarium19 & CIFAR10 & CIFAR100 & ImageNet-100 \\
         \hline
$|Y^L|$ & 100  & 98             & 341         & 5       & 80       & 50           \\
$|D^L|$ & 1.5K & 2.0K           & 8.9K        & 12.5K   & 20K      & 31.9K        \\ \hline
$|Y^U|$ & 200  & 196            & 683         & 10      & 100      & 100          \\
$|D^U|$ & 4.5K & 6.1K           & 25.4K       & 37.5K   & 30K      & 95.3K  \\
 \bottomrule
\end{tabular}%
\end{table}

\section{More results and discussions}
\label{app:more results}

\textbf{More visualization for the confusion issue.}
\cref{fig:hist} visualizes the confusion issue in terms of density histogram. It is clear that without the constraints on the independence assumption between known and unknown classes, the latent feature distribution of $93$ overlaps with known classes. Our approach significantly alleviates this issue.  

\begin{figure}[t]
  \centering
  \includegraphics[width=0.9\linewidth]{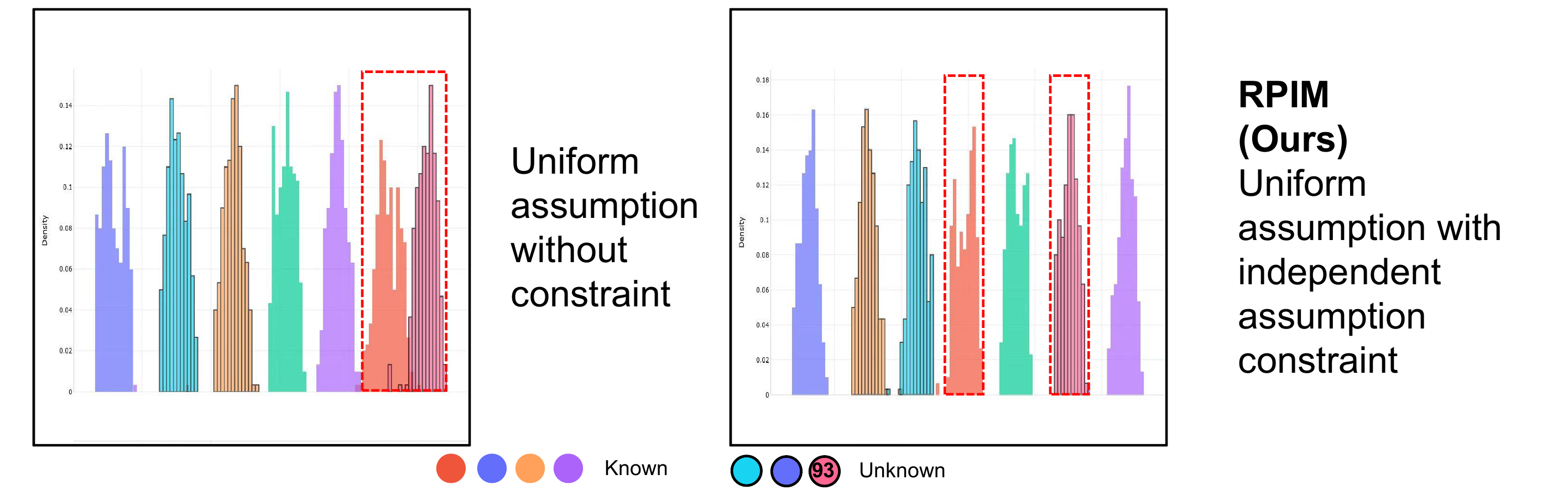}
  \caption{Density histogram of latent features $Z$ of unlabeled samples from the known and unknown classes on CIFAR100. 
  }
  \label{fig:hist}
\end{figure}

\begin{table}[ht]
\caption{Ablation results: Accuracy scores across fine-grained and generic datasets of each setting. Note that all other experiments use \textbf{fixed weight decay=$0.01$}. 
The best results are highlighted in \textbf{bold}.
Improvement and degradation in our approach from baseline (PIM) are highlighted in \red{red$\uparrow$} and \blue{blue$\downarrow$}, respectively.
$h$ denotes the proposed semantic transformation.
It can be seen that even using fixed weight decay, our RPIM still brings competitive results. 
}
\label{tab:ablation_res_use_fixed_wd}
\resizebox{\textwidth}{!}{%
\begin{tabular}{lccc|ccc|ccc|ccc}
\toprule
 &  &  &  & \multicolumn{3}{c|}{\textbf{CUB}}     & \multicolumn{3}{c|}{\textbf{Stanford Cars}} & \multicolumn{3}{c}{\textbf{Herbarium19}} \\ \hline
Settings & $h$ & $\mathcal{L}_{R}$ & $\mathcal{L}_{S}$ &  \multicolumn{1}{c}{{\;\;\;\;\;\;All\;\;\;\;\;\;}} & \multicolumn{1}{c}{{Known}} & \multicolumn{1}{c|}{{Unknown}} & \multicolumn{1}{c}{{\;\;\;\;\;\;All\;\;\;\;\;\;}} & \multicolumn{1}{c}{{Known}} & \multicolumn{1}{c|}{{Unknown}} & \multicolumn{1}{c}{{\;\;\;\;\;\;All\;\;\;\;\;\;}} & \multicolumn{1}{c}{{Known}} & \multicolumn{1}{c}{{Unknown}} \\ \hline
Baseline ($\mathcal{L}_{pim}$) &  &  &  & 62.7 & 75.7 & 56.2 & 43.1 & 66.9 & 31.6 & 42.3 & 56.1 & 34.8 \\ \hline
$\mathcal{L}_{pim}$+$\mathcal{L}_{R}$ &  & \checkmark &  & 62.5 & 76.3 & 55.6 & 43.9 & 64.8 & 33.7 & 42.1 & 55.8 & 34.8 \\
$\mathcal{L}_{pim}$+$\mathcal{L}_{S}$ &  &  & \checkmark & 60.2 & 72.8 & 53.9 & 42.1 & 66.4 & 30.4 & 42.2 & 56.4 & 34.5 \\
$\mathcal{L}_{pim}$+$\mathcal{L}_{R}$+$\mathcal{L}_{S}$ &  & \checkmark & \checkmark & 61.3 & 73.4 & 55.3 & 42.6 & 65.5 & 31.5 & 41.7 & 56.1 & 33.9 \\
\hline
\multicolumn{3}{c}{} & \multicolumn{9}{c}{Using semantic transformation} \\ \hline
$\mathcal{L}_{pim}$+$h$ & \checkmark &  &  & 64.9 & 76.7 & 58.9 & 44.7 & 65.8 & 34.6 & 43 & 57.4 & 35.2 \\
$\mathcal{L}_{pim}$+$h$+$\mathcal{L}_{R}$ & \checkmark & \checkmark &  & 66.3 & 76.2 & 61.3 & 44.4 & 65.4 & 34.3 & 43 & 56.6 & 35.7 \\
$\mathcal{L}_{pim}$+$h$+$\mathcal{L}_{S}$ & \checkmark &  & \checkmark & 64.9 & 75.5 & 59.7 & \textbf{45.8} & 66.7 & \textbf{35.7} & \textbf{43.2} & \textbf{57.4} & \textbf{35.6} \\
\rowcolor{mygray} \textbf{Ours} ($\mathcal{L}_{pim}$+$h$+$\mathcal{L}_{R}$+$\mathcal{L}_{S}$) & \checkmark & \checkmark & \checkmark & \textbf{66.8}~\red{($4.1\uparrow$)} & \textbf{77.3}~\red{($1.6\uparrow$)} & \textbf{61.5}~\red{($5.3\uparrow$)} & \textbf{45.8}~\red{($2.7\uparrow$)} & \textbf{67.5}~\red{($0.6\uparrow$)} & 35.3~\red{($3.7\uparrow$)} & 43.0~\red{($0.7\uparrow$)} & \textbf{57.4}~\red{($1.4\uparrow$)} & 35.2~\red{($0.4\uparrow$)} \\ \hline
 &  &  & & \multicolumn{3}{c|}{\textbf{CIFAR10}} & \multicolumn{3}{c|}{\textbf{CIFAR100}}      & \multicolumn{3}{c}{\textbf{ImageNet-100}}  \\ \hline
Settings & $h$ & $\mathcal{L}_{R}$ & $\mathcal{L}_{S}$ & \multicolumn{1}{c}{{\;\;\;\;\;\;All\;\;\;\;\;\;}} & \multicolumn{1}{c}{{Known}} & \multicolumn{1}{c|}{{Unknown}} & \multicolumn{1}{c}{{\;\;\;\;\;\;All\;\;\;\;\;\;}} & \multicolumn{1}{c}{{Known}} & \multicolumn{1}{c|}{{Unknown}} & \multicolumn{1}{c}{{\;\;\;\;\;\;All\;\;\;\;\;\;}} & \multicolumn{1}{c}{{Known}} & \multicolumn{1}{c}{{Unknown}} \\ \hline
Baseline ($\mathcal{L}_{pim}$) &  &  &  & 94.7 & 97.4 & 93.3 & 78.3 & 84.2 & 66.5 & \textbf{83.1} & 95.3 & \textbf{77} \\ \hline
$\mathcal{L}_{pim}$+$\mathcal{L}_{R}$ &  & \checkmark &  & 94.7 & 97.4 & 93.4 & 78.3 & 84.2 & 66.6 & 83.1 & 95.3 & 77 \\
$\mathcal{L}_{pim}$+$\mathcal{L}_{S}$ &  &  & \checkmark & 95.4 & 97.3 & 94.4 & 77.4 & 84.7 & 62.8 & 81.1 & 95.5 & 73.9 \\
$\mathcal{L}_{pim}$+$\mathcal{L}_{R}$+$\mathcal{L}_{S}$ &  & \checkmark & \checkmark & 95.4 & 97.3 & 94.5 & 77.4 & 84.6 & 62.8 & 81.1 & 95.5 & 73.9 \\
\hline
\multicolumn{3}{c}{} & \multicolumn{9}{c}{Using semantic transformation} \\ \hline
$\mathcal{L}_{pim}$+$h$ & \checkmark &  &  & 95 & 97.5 & 93.7 & 78.8 & 84.1 & 68.2 & 80.4 & 95.3 & 72.9 \\
$\mathcal{L}_{pim}$+$h$+$\mathcal{L}_{R}$ & \checkmark & \checkmark &  & 94.9 & 97.5 & 93.6 & \textbf{78.9} & 84.1 & \textbf{68.6} & 82.2 & 95.3 & 75.7 \\
$\mathcal{L}_{pim}$+$h$+$\mathcal{L}_{S}$ & \checkmark &  & \checkmark & 95.5 & \textbf{97.7} & 94.4 & 78.5 & \textbf{84.6} & 66.3 & 81.4 & \textbf{95.5} & 74.3 \\
\rowcolor{mygray} \textbf{Ours} ($\mathcal{L}_{pim}$+$h$+$\mathcal{L}_{R}$+$\mathcal{L}_{S}$) & \checkmark & \checkmark & \checkmark & \textbf{95.6}~\red{($0.9\uparrow$)} & \textbf{97.7}~\red{($0.3\uparrow$)} & \textbf{94.5}~\red{($1.2\uparrow$)} & 78.6~\red{($0.3\uparrow$)} & \textbf{84.6}~\red{($0.4\uparrow$)} & 66.7~\red{($0.2\uparrow$)} & 81.4~\blue{($1.7\downarrow$)} & \textbf{95.5}~\red{($0.2\uparrow$)} & 74.3~\blue{($2.7\downarrow$)} \\
\bottomrule
\end{tabular}%
}
\end{table}

\textbf{Influence of tuned weight decay.}
\cref{tab:ablation_res} exhibits the results across all datasets using tuned weight decay, except the baseline rows. Combining \cref{tab:ablation_res} and \cref{tab:ablation_res_use_fixed_wd}, it shows that the tuned weight decay leads to general improvements across all datasets and all settings. Notably, our approach leads to significant improvements in both using fixed and tuned weight decay, further validating RPIM's feasibility.

\textbf{Sensitive analysis.}
\cref{tab:sensitive} exhibits results across all datasets using different values of $\gamma$. It can be seen that the results across all settings are stable, especially the average results across all datasets, indicating that our approach is not very sensitive to the hyper-parameter $\gamma$. 
These results also further validate the certified efficacy of $\mathcal{L}_{R}$ as proofed in \cref{prop:our_better_old}.

\textbf{Effect of seed.} Our method is not sensitive to the value of seeds. We try different seeds, including $[0,1,2,3,4,5]$, and all the results are the same.

\textbf{Visualization of semantic-bias transformations.}
Figure \ref{fig:transformations} further elucidates that the default linear layers, whether biased or not, often fail to establish clear decision boundaries between classes. In contrast, our method consistently facilitates improvements.

\textbf{Comparing to more transformations.} 
We stack more linear layers 
to refine the latent features to validate their efficacy. As shown in \cref{tab:mlp_tranformtions}, more layers may lead to severe semantic collapse, resulting in dramatic performance degradation.

\begin{figure}[t]
  \centering
  \includegraphics[width=0.77\linewidth]{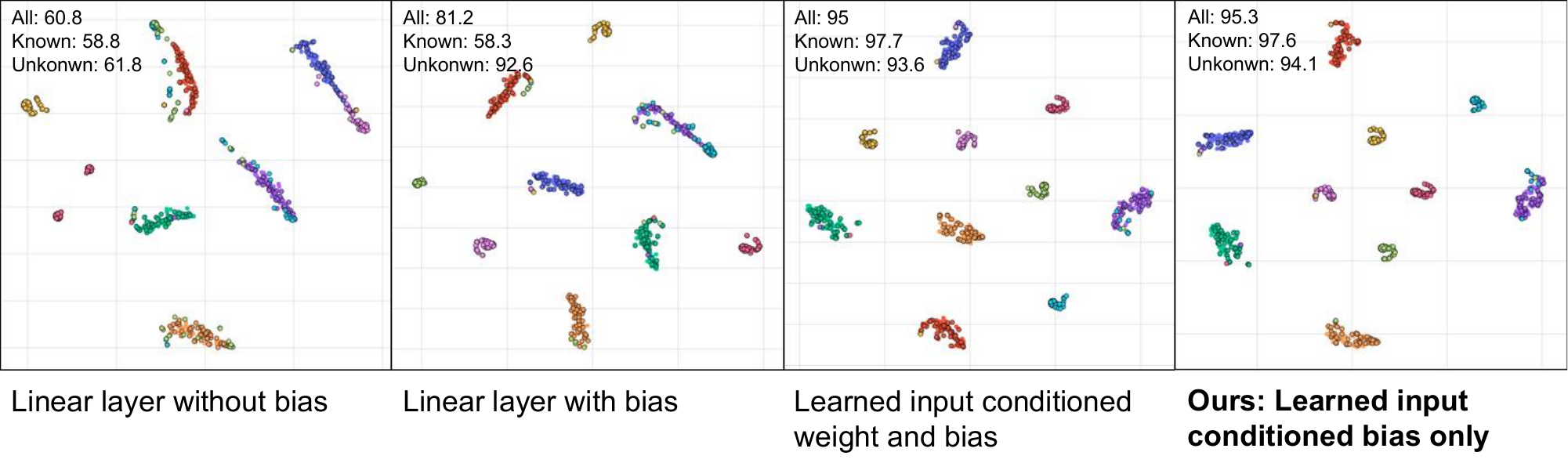}
  \caption{T-SNE map of unlabeled data latent features $Z^L$ from models that use different transformations trained on CIFRA10 dataset. Different colors represent different classes. It can be seen that our proposed semantic-bias transformation leads to the best results.  
  }
    \label{fig:transformations}
\end{figure}

\begin{table}[t]
\caption{Sensitive analysis: Results of using different values for $\eta$ for weighting $\mathcal{L}_{R}$.
Our used final value $\eta=0.03$ is highlighted. It can be seen that PRIM is not sensitive to the value of $\eta$.
}
\label{tab:sensitive}
\resizebox{\textwidth}{!}{%
\begin{tabular}{l|ccc|ccc|ccc|ccc}
\toprule
\multicolumn{1}{l|}{Setting} & \multicolumn{3}{c|}{\textbf{Average}} & \multicolumn{3}{c|}{\textbf{CIFAR10}} & \multicolumn{3}{c|}{\textbf{CIFAR100}} & \multicolumn{3}{c}{\textbf{ImageNet-100}} \\ \hline
$\eta$ for $\mathcal{L}_{R}$ & \multicolumn{1}{c}{\textbf{\;\;\;\;All\;\;\;\;}} & \multicolumn{1}{c}{\textbf{\;Known\;}} & \multicolumn{1}{c|}{\textbf{\;Unknown\;}} & \multicolumn{1}{c}{\textbf{\;\;\;\;All\;\;\;\;}} & \multicolumn{1}{c}{\textbf{\;Known\;}} & \multicolumn{1}{c|}{\textbf{\;Unknown\;}} & \multicolumn{1}{c}{\textbf{\;\;\;\;All\;\;\;\;}} & \multicolumn{1}{c}{\textbf{\;Known\;}} & \multicolumn{1}{c|}{\textbf{\;Unknown\;}} & \multicolumn{1}{c}{\textbf{\;\;\;\;All\;\;\;\;}} & \multicolumn{1}{c}{\textbf{\;Known\;}} & \multicolumn{1}{c}{\textbf{\;Unknown\;}} \\ \hline
0.01 & 51.4 & 66.5 & 43.8 & 64.8 & 75.2 & 59.6 & 46.2 & 67.2 & 36.1 & 43.2 & 57.2 & 35.6 \\
0.02 & 51.9 & 67.3 & 44.1 & 66.7 & 77.2 & 61.4 & 45.9 & 67.4 & 35.5 & 43.2 & 57.4 & 35.5 \\
\rowcolor{mygray}0.03 & 51.9 & 67.4 & 44.0 & 66.8 & 77.3 & 61.5 & 45.8 & 67.5 & 35.3 & 43 & 57.4 & 35.2 \\
0.04 & 51.8 & 66.8 & 44.2 & 66.5 & 76.3 & 61.7 & 45.7 & 66.7 & 35.5 & 43.2 & 57.4 & 35.5 \\
0.05 & 51.4 & 66.9 & 43.6 & 65.7 & 76.7 & 60.2 & 45.8 & 66.4 & 35.9 & 42.8 & 57.5 & 34.8 \\ \hline
\multicolumn{1}{l|}{Setting} & \multicolumn{3}{c|}{\textbf{Average}} & \multicolumn{3}{c|}{\textbf{CIFAR10}} & \multicolumn{3}{c|}{\textbf{CIFAR100}} & \multicolumn{3}{c}{\textbf{ImageNet-100}} \\ \hline
$\eta$ for $\mathcal{L}_{R}$ & \multicolumn{1}{c}{\textbf{\;\;\;\;All\;\;\;\;}} & \multicolumn{1}{c}{\textbf{\;Known\;}} & \multicolumn{1}{c|}{\textbf{\;Unknown\;}} & \multicolumn{1}{c}{\textbf{\;\;\;\;All\;\;\;\;}} & \multicolumn{1}{c}{\textbf{\;Known\;}} & \multicolumn{1}{c|}{\textbf{\;Unknown\;}} & \multicolumn{1}{c}{\textbf{\;\;\;\;All\;\;\;\;}} & \multicolumn{1}{c}{\textbf{\;Known\;}} & \multicolumn{1}{c|}{\textbf{\;Unknown\;}} & \multicolumn{1}{c}{\textbf{\;\;\;\;All\;\;\;\;}} & \multicolumn{1}{c}{\textbf{\;Known\;}} & \multicolumn{1}{c}{\textbf{\;Unknown\;}} \\ \hline
0.01 & 87.0 & 92.8 & 82.7 & 95.1 & 97.7 & 93.8 & 82.7 & 85.5 & 77.1 & 83.2 & 95.2 & 77.2 \\
0.02 & 87.0 & 92.8 & 82.7 & 95.2 & 97.7 & 94 & 82.6 & 85.5 & 76.9 & 83.2 & 95.2 & 77.2 \\
\rowcolor{mygray}0.03 & 87.1 & 92.7 & 82.8 & 95.3 & 97.6 & 94.1 & 82.7 & 85.4 & 77.2 & 83.2 & 95.2 & 77.2 \\
0.04 & 87.0 & 92.8 & 82.7 & 95.2 & 97.7 & 94 & 82.5 & 85.4 & 76.8 & 83.2 & 95.2 & 77.2 \\
0.05 & 87.0 & 92.8 & 82.7 & 95.2 & 97.7 & 94 & 82.6 & 85.5 & 76.9 & 83.2 & 95.2 & 77.2 \\
\bottomrule
\end{tabular}%
}
\end{table}


\begin{table}
\centering
\caption{Semantic transformation: Comparison between different transformations. All results better than the baseline are highlighted in \textbf{bold}. }
\label{tab:mlp_tranformtions}
\resizebox{\textwidth}{!}{%
\begin{tabular}{l|ccc|ccc|ccc} 
\toprule
                                                   & \multicolumn{3}{c|}{\textbf{CUB}}             & \multicolumn{3}{c|}{\textbf{Stanford Cars}}   & \multicolumn{3}{c}{\textbf{Herbarium19}}       \\ 
\midrule
Setting                                            & \;\;\;\;All\;\;\;\;           & Known         & Unknown       & \;\;\;\;All\;\;\;\;           & Known         & Unknown       & \;\;\;\;All\;\;\;\;           & Known         & Unknown        \\ 
\midrule
Baseline                                           & 62.7          & 75.7          & 56.2          & 43.1          & 66.9          & 31.6          & 42.3          & 56.1          & 34.8           \\ 
\hline
Linear layer with bias & 52.6 & 67.2 & 45.3 & \textbf{45.6} & 66.4 & \textbf{35.6} & 40.9 & 55.6 & 33.0 \\

Two linear layers   with bias                                   & 0.7           & 0.0           & 1.0           & 1.1           & 0.0           & 1.6           & 2.2           & 0.0           & 3.4            \\

\rowcolor{mygray}\textbf{Ours}: Learned input conditioned bias only & \textbf{66.8} & \textbf{77.3} & \textbf{61.5} & \textbf{45.8} & \textbf{67.5} & \textbf{35.3} & \textbf{43.0} & \textbf{57.4} & \textbf{35.2}  \\ 
\midrule
                                                   & \multicolumn{3}{c|}{\textbf{CIFAR10}}         & \multicolumn{3}{c|}{\textbf{CIFAR100}}        & \multicolumn{3}{c}{\textbf{ImageNet-100}}      \\ 
\midrule
Setting                                            & All           & Known         & Unknown       & All           & Known         & Unknown       & All           & Known         & Unknown        \\ 
\midrule
Baseline                                           & 94.7          & 97.4          & 93.3          & 78.3          & 84.2          & 66.5          & 83.1          & 95.3          & 77             \\ 
\hline
Linear layer with bias & 81.2 & 58.3 & 92.6 & 77.6 & \textbf{84.7} & 63.4 & \textbf{85.1} & 95.2 & \textbf{80} \\

Two linear layers with bias                                 & 13.3          & 0.0           & 20.0          & 1.7           & 0.0           & 5.0           & 1.4           & 0.0           & 2.1            \\

\rowcolor{mygray}\textbf{Ours}: Learned input conditioned bias only & \textbf{95.3} & \textbf{97.6} & \textbf{94.1} & \textbf{82.7} & \textbf{85.4} & \textbf{77.2} & \textbf{83.2} & 95.2          & \textbf{77.2}  \\
\bottomrule

\end{tabular}
}
\end{table}

\end{document}